%% file: main.tex
\documentclass{article}

\usepackage{microtype}
\usepackage{graphicx}
\usepackage{booktabs} 

\usepackage{hyperref}

\usepackage[accepted]{icml2024}

\usepackage{amsmath}
\usepackage{amssymb}
\usepackage{mathtools}
\usepackage{amsthm}

\input{utils}

\usepackage[capitalize,noabbrev]{cleveref}

\theoremstyle{plain}
\newtheorem{theorem}{Theorem}[section]
\newtheorem{proposition}[theorem]{Proposition}

\theoremstyle{definition}
\newtheorem{definition}[theorem]{Definition}

\theoremstyle{remark}

\usepackage[textsize=tiny]{todonotes}

\icmltitlerunning{Improving upon the improved Precision Recall metric for assessing generative models}

\begin{document}

\twocolumn[
\icmltitle{Unifying and extending Precision Recall metrics for assessing generative models}



\icmlsetsymbol{equal}{*}

\begin{icmlauthorlist}
\icmlauthor{Benjamin Sykes}{equal,yyy}
\icmlauthor{Loïc Simon}{equal,yyy}
\icmlauthor{Julien Rabin}{equal,yyy}

\end{icmlauthorlist}

\icmlaffiliation{yyy}{University of Caen Normandie, ENSICAEN, CNRS, France}

\icmlcorrespondingauthor{Benjamin Sykes}{benjamin.sykes@unicaen.fr}

\icmlkeywords{Machine Learning, ICML}

\vskip 0.3in
]




\input{content/1_abstract}
\input{content/1_intro_new}

\input{content/2_theory_recap}

\input{content/3_extensions}

\input{content/4_experiments}

\input{content/5_discussion}

\input{content/6_conclusion}

\clearpage

\bibliography{bibliography}
\bibliographystyle{icml2024}

\clearpage
\input{content/9_appendix}


\end{document}

%% file: utils.tex
\usepackage{soul}           
\usepackage{mathrsfs}       
\usepackage{dsfont}
\usepackage{amsthm}
\usepackage{amsmath}
\usepackage{amsfonts}
\usepackage{amssymb}
\usepackage{subcaption}
\usepackage{bm,array}

\usepackage{multirow}

\newcommand{\N}{\mathbb{N}}
\newcommand{\R}{\mathbb{R}}
\renewcommand{\P}{\mathbb{P}}

\newcommand{\juju}[1]{}

\newcommand{\measSpace}{\Omega}

\newcommand{\probaSet}{\mathcal M_p(\measSpace)}
\newcommand{\AC}{\text{AC}}
\newcommand{\PRD}{\text{PRD}}
\newcommand{\dPRD}{\partial\PRD}
\newcommand{\supp}{\mathrm{supp}}
\newcommand{\cosupp}{\mathrm{cosupp}}
\newcommand{\fpr}{\mathrm{fpr}}
\newcommand{\fnr}{\mathrm{fnr}}
\newcommand{\F}{\mathscr{F}}
\newcommand{\E}{\mathbb{E}}
\newcommand{\1}{\mathds{1}}
\newcommand{\BKNN}{B_{kNN}}
\newcommand{\X}{{\mathcal X}}

\DeclareMathOperator*{\argmin}{arg\,min}

%% file: content/1_abstract.tex
\begin{abstract}
With the recent success of generative models in image and text, the evaluation of generative models has gained a lot of attention. 
Whereas most generative models are compared in terms of scalar values such as Fréchet Inception Distance (FID) or Inception Score (IS), in the last years \cite{sajjadi_AssessingGenerativeModels_2018} proposed a definition of 
precision-recall curve to characterize the closeness of two distributions.
Since then, various approaches to precision and recall have seen the light \cite{kynkaanniemi_ImprovedPrecisionRecall_2019,naeem_ReliableFidelityDiversity_2020,park_ProbabilisticPrecisionRecall_2023}. 
They center their attention on the extreme values of precision and recall, but apart from this fact, their ties are elusive.
In this paper, we unify most of these approaches under the same umbrella, relying on the work of \cite{simon_RevisitingPrecisionRecall_2019}. 
Doing so, we were able not only to recover entire curves, but also to expose the sources of the accounted pitfalls of the concerned metrics. 
We also provide consistency results that go well beyond the ones presented in the corresponding literature. 
Last, we study the different behaviors of the curves obtained experimentally.
\end{abstract}

%% file: content/1_intro_new.tex
\section{Introduction}
In this article, we consider metrics designed to evaluate the adequacy of a generative model to the distribution it is assumed to capture. 
In itself, this problem consists of evaluating the closeness of the real distribution, hereafter denoted by $P$ and the generated one, denoted by $Q$.
In an early period, several scalar metrics were designed such as Inception Score \cite{salimans_ImprovedTechniquesTraining_2016} and the iconic Fréchet Inception Distance \cite{heusel_GANsTrainedTwo_2017}. 
A rich literature completes this line of research, pointing at limitations and extensions or concurrent scalar metrics. 
Notwithstanding, one pitfall is shared by any such scalar metric in that they cannot account separately for two types of failures:  namely for the lack of realism (a.k.a fidelity), and the lack of variability (diversity). 
This assessment was first carried out by \cite{sajjadi_AssessingGenerativeModels_2018} where the authors developed a trade-off curve known as the Precision Recall Curve, which characterizes both types of flaws. 
Each point of the curve has two components $\alpha$ (a.k.a precision) and $\beta$ (a.k.a recall) which in essence represent respectively the mass of $P$ and $Q$ that can be extracted simultaneously by selecting a subset of their common support (the formal definition will be clarified later on). This intuitive description is illustrated in Fig.~\ref{fig:PR}.

\begin{figure}[!t]
    \centering
         \includegraphics[width=0.48\textwidth]{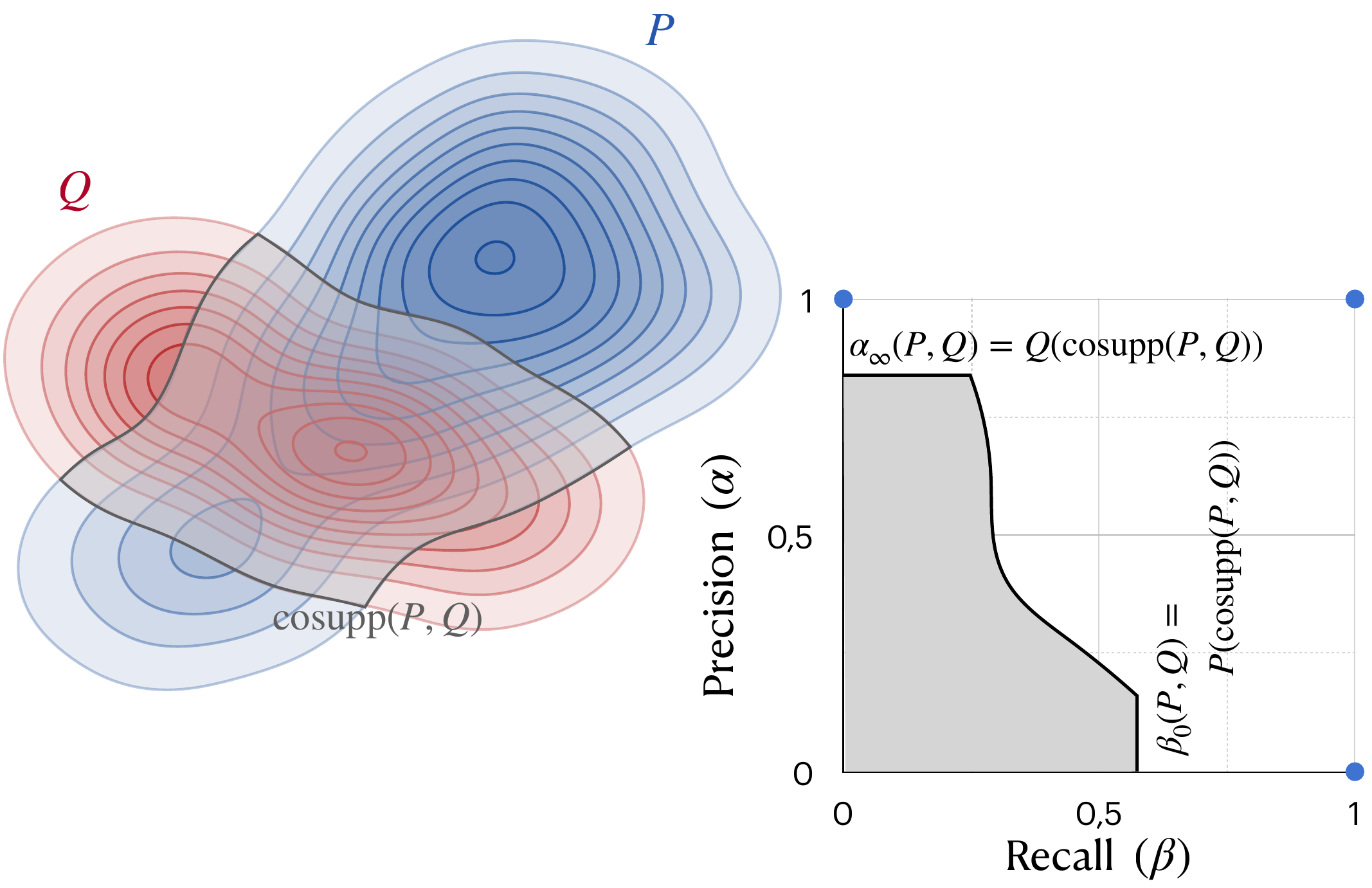}
    \caption{
    Right: the PR-curve is the frontier of the shaded area composed of all admissible PR pairs $(\beta,\alpha)$. In essence, these pairs represent the mass of $P$ and $Q$ that one can recover by selecting a subset of the common support (gray area on the left). More precisely, by selecting regions of high likelihood of $P$, one trades precision ($\alpha$) in favor of recall ($\beta$).  The extreme values $\beta_0(P,Q)$ and  $\alpha_\infty(P,Q)$ embody the respective masses of the entire common support.}
    \label{fig:PR}
\end{figure}

The authors of \cite{sajjadi_AssessingGenerativeModels_2018} give a formal definition of their curve for general distributions, but they were able to provide a practical characterization (amenable to empirical evaluation) only in the case of discrete distributions, and relied on clustering to get an algorithmic evaluation of the curve. 
Followup theoretical insights were provided in \cite{simon_RevisitingPrecisionRecall_2019} where an alternate characterization was exposed which extend to general distributions. 
More details will be given later on since it will be central in our developments.
This work was also complemented by a deep theoretical analysis of the link between this curve and many other statistical notions in \cite{siry_TheoreticalEquivalenceSeveral_2022}. In particular, the authors show that PR curves are in fact equivalent to Divergence Frontiers, which were developed in \cite{djolonga_PrecisionRecallCurvesUsing_2020} in an attempt to generalize PR curves.

In parallel to these theoretical works, a handful of practical metrics have been developed such as \cite{kynkaanniemi_ImprovedPrecisionRecall_2019, naeem_ReliableFidelityDiversity_2020, cheema_PrecisionRecallCover_2023, khayatkhoei_EmergentAsymmetryPrecision_2023, park_ProbabilisticPrecisionRecall_2023}.
All of them point to the shortcomings of earlier attempts and propose an alternative to improve the concerned aspects. 
A few remarks can be made already about most of these approaches. Starting from \cite{kynkaanniemi_ImprovedPrecisionRecall_2019}, it was argued that evaluating the extreme Precision and Recall was enough in practice, and therefore, instead of extracting a whole curve all of these variants only evaluates two scalar metrics: namely the extreme precision (denoted by $\alpha_\infty(P,Q)$) and the extreme recall ($\beta_0(P,Q)$).
This choice may appear justified, but in many cases, the theoretical values of both metrics reach their saturation level (i.e. $1$) even though the two distributions $P$ and $Q$ differ substantially. 
As a result, the metrics do not provide much insight into the closeness of $P$ and $Q$.

Besides, the relation between the empirical estimates and the associated theoretical metrics (i.e. at the population level) is not always clear. 
In particular, oftentimes, experimental behaviors praised by the authors are in fact contradictory with the expected behavior of the theoretical metrics. 
A typical example concerns experiments where $P$ and $Q$ are both Gaussian distributions but where $Q$ is shifted away from $P$. 
In that case, the theoretical metrics remain constant whatever the shift, while the empirical estimates appear to decrease with the magnitude of the shift. 
The authors embrace this desirable experimental observation without recognizing that it arises from the compensatory interaction of two underlying flaws: on the one hand, the two theoretical metrics are lacking, and on the other hand, their empirical estimators are not accurate.
Instead, we advocate for extracting entire PR curves since this gives a complete picture of the disparity between the two distributions, and for conducting a thorough analysis of the consistency of the empirical estimator w.r.t to the population level counterparts. Note that in the literature, consistency is at best studied in the case $P=Q$.

In fact, we make the following contributions. 
First, we show that most previous fidelity (resp. diversity) 
metrics can be interpreted as estimators of $\alpha_\infty(P,Q)$ (resp. $\beta_0(P,Q)$) thanks to the binary classification point of view developed in \cite{simon_RevisitingPrecisionRecall_2019}.
As a result, we show that they can be extended into complete curves in quite natural ways.
Besides, concurrent approaches differ merely by their underlying hypothesis class (that is the family of classifiers). 
For instance, the Improved Precision-Recall (IPR) of \cite{kynkaanniemi_ImprovedPrecisionRecall_2019} relates to kernel density estimators with adaptive bandwidth, while coverage \cite{naeem_ReliableFidelityDiversity_2020} (and its precision counterpart obtained by swapping the role of $P$ and $Q$ \cite{khayatkhoei_EmergentAsymmetryPrecision_2023}) relates to knn classification\footnote{On the contrary we will not pursue the same endeavor for \cite{park_ProbabilisticPrecisionRecall_2023} since their metric does not bear any reminiscence to classical non-parametric classification literature.}. 
Observed under this lens, any idiosyncrasy of the emerging hypothesis class (or its usage)  stands out and can be readily amended. 
In particular, one common pitfall transpires: namely the absence of data split between training samples (used to fit the classifier) and test samples (used to evaluate precision or recall). 
This introduces a negative bias in the estimators as well as correlations that make the analysis of the estimator consistency unnecessarily challenging.
On the contrary, by using split, which is standard practice, the bias of the estimators is trivially positive and consistency can be studied using standard techniques. 
We actually provide such a result for the curve associated with coverage (which goes beyond the $P=Q$ case considered originally in \cite{naeem_ReliableFidelityDiversity_2020}).
We conduct experiments on toy examples similar to those promoted in previous works to re-assess the pros and cons of the PR curve estimator variants (with and without fixes). Our conclusion in this regard is that our extension to coverage has better results than the other extensions. While iPR is sensitive to outliers and is less efficient than our other extensions with the original setting, we provide amendments that correct this behavior, namely setting the value of neighbors $k$ in kNN classifier to grow with the number of data points. With our extended metrics, we also provide some illustrations of the known cases where a generative model either creates, drops, or re-weights modes. Precision and recall curves allow to finely represent those three simultaneous behaviors at once.

%% file: content/2_theory_recap.tex
\section{Recap on the relevant literature}
Let $\measSpace$ a measurable space, and denoting $\probaSet$ the set of distributions over $\measSpace$, let $P,Q\in\probaSet$ (e.g. real and generated distribution). 

\subsection{The gist on the original PR curve notion}
First the PR set between $P$ and $Q$ is defined as the set $\PRD(P,Q)$ of non negative couples $(\alpha, \beta)$ such that, 
\iffalse $\exists \mu\in\AC(P,Q)$ (meaning a distribution that is absolutely continuous w.r.t $P$ and $Q$) 
\else $\exists \mu\in\probaSet$ 
\fi 
verifying both $P\geq \beta \mu$ and $Q\geq \alpha \mu$. In a nutshell, the two conditions translate the fact that the ''probe'' distribution $\mu$ can simultaneously ``extract''  some mass $\beta$ from $P$ and $\alpha$ from $Q$.
Note that the PR set is included within $[0,1]^2$ and it is a cone, meaning that $\forall (\alpha, \beta)\in \PRD(P,Q)$ and $0\leq \gamma\leq 1$ then $(\gamma\alpha, \gamma \beta)\in\PRD(P,Q)$. As a result, this set is characterized by its (upper-right) Pareto frontier denoted by $\dPRD(P,Q)$ which can be parameterized as $\dPRD(P,Q) = \{(\alpha_\lambda,\beta_\lambda), \lambda \in\bar\R^+\}$ with 
\begin{equation}
\label{eq:primalPR}
\begin{split}
\alpha_\lambda &= (\lambda P \wedge Q)(\omega)\\
\beta_\lambda &= \tfrac{\alpha_\lambda}{\lambda}
\end{split}
\end{equation}
where $\wedge$ is the minimum operator between two measures (see \cite{simon_RevisitingPrecisionRecall_2019} for details).

This whole curve captures both extreme precision and recall values corresponding to $\alpha_\infty$ and $\beta_0$ which play a central role in the later literature starting from \cite{kynkaanniemi_ImprovedPrecisionRecall_2019}. In addition, it also describes how similarly the mass is distributed within the common support of $P$ and $Q$ (see \cite{siry_TheoreticalEquivalenceSeveral_2022} for details). 
Interestingly, this curve can be also characterized in a dual way, based on a specific two-sample classification problem \cite{simon_RevisitingPrecisionRecall_2019}.
In short, for a sample $Z=U X +(1-U) Y \sim \frac 12 (P+Q)$ (that is $U$ is coin flip, $X\sim P$ and $Y\sim Q$), the task consists in predicting whether $U=1$ (i.e. $Z=X\sim P$). Then,
\begin{equation}
\label{eq:dualPR}
\begin{split}
    \alpha_{\lambda}(P,Q) &= \min_{f \in \F} \left\{ \lambda \cdot \fpr(f) + \fnr(f) \right\} \\
    \beta_{\lambda}(P,Q) &= \min_{f \in \F} \left\{ \fpr(f) + \frac{\fnr(f)}{\lambda} \right\} 
\end{split}
\end{equation}
where the hypothesis class $\F$ is composed of all binary classifiers on $\Omega$, and $\fnr(f)$ (resp. $\fpr(f)$) represents the false negative (resp. positive) rate of the classifier $f$, that is to say the probability that a sample $Y\sim Q$ was classified as a sample from $P$ (resp. vice versa). More precisely, 
\begin{equation}
    \label{eq:fpr}
    \fpr(f) = \int 1-f dP \text{ and } \fnr(f)=\int f dQ
\end{equation}

\subsection{Re-assessing extreme precision-recall values}

In the literature, the accepted expression of the extreme precision is $\alpha_\infty(P,Q):=\lim_{\lambda\to\infty} \alpha_\lambda(P,Q) = Q(\supp(P))$. 
In fact, this identity is flimsy and requires to be amended mainly because the support of a distribution is defined up to null sets for that distribution. 
In the incriminated identity, the issue stems from the fact that adding a $P$-null set can change the $Q$-mass of the set, and therefore the right-hand side is not well characterized.
Let us clarify a few notions.
\begin{definition}[support and co-support]
Let $A$ be a measurable subset of $\measSpace$.
We say that $A$ is a
\begin{itemize}
    \item support of $P$, denoted\footnote{This is a slight abuse of notations.}  $A=\supp(P)$, iff  $P(A^c)=0$. 
    \item co-support of $P$ and $Q$  denoted $A=\cosupp(P,Q)$ iff 
    (
    $(P\wedge Q)(A^c)=0$ 
    and $\forall B\subset A$, $P(B)=0\Leftrightarrow Q(B)=0$)
\end{itemize}
\end{definition}
As the reader may notice, the second notion is characterized up to sets that are simultaneously $P$ and $Q$ null. More precisely, we have the following result.
\begin{proposition}
    \label{thm:cosupp}
    Let $P, Q$ two distributions. Then all co-supports of $P$ and $Q$ have the same $Q$-mass and $$\alpha_\infty(P,Q)=Q(\cosupp(P,Q))$$
\end{proposition}
\begin{proof}(Sketch) 
First, if $A, A'$ are two co-supports. Then $Q(A)=Q(A\cap A') = Q(A')$. Indeed, if $Q(A)\geq Q(A\cap A')$ then letting $B=A\setminus A'\subset A$ and $Q(B)>0$. Yet $B\subset A'^c$ so that $P(B)\leq P(A'^c)=0$ yielding a contradiction.

Second, let us exhibit a co-support $C$ that verifies $Q(C)=\alpha_\infty$.
In \cite{simon_RevisitingPrecisionRecall_2019}, it is shown that Eq.~\ref{eq:dualPR} can be restated as $\alpha_\infty = \min_{A\text{ s.t. } P(A^c)=0} Q(A)$. 
Let $A^*$ one of the minimizers. Without further care, $A^*$ should be merely a particular support of $P$ but could still not be a co-support. 
We therefore need to filter out any part of the space that charges $P$ but not $Q$ (which will make it a co-support without affecting its $Q$-mass).
To do so, we consider $C = A^*\setminus \cap_{\lambda>0} \{\lambda P > Q\}$. 
First, the monotone convergence theorem implies that $Q(C)=Q(A^*)=\alpha_\infty$. 

It remains to show that $C$ is indeed a co-support.
Notice that $P\wedge Q(C^c)=(P\wedge Q)({A^*}^c\bigcup \cap_{\lambda>0} \{\lambda P > Q\})\leq P({A^*}^c) + Q(\cap_{\lambda>0} \{\lambda P > Q\})=0$ (because of the constraint on $A^*$ for the first summand, and by the monotone convergence theorem again for the other summand).

Besides let $B\subset C \subset (\cap_{\lambda>0} \{\lambda P > Q\})^c=\cup_{\lambda>0} \{\lambda P \leq Q\}$ so that $Q(B)=0 \implies P(B)=0$.
Conversely, if $P(B)=0$ let us show that $Q(B)=0$. 
To do so let us reason by contradiction, by assuming that $Q(B)>0$. Then $A=A^*\setminus B$ verifies the constraint $P(A)=0$ and $Q(A)=Q(A^*)-Q(B)$ (because $B\subset C\subset A^*$). Then $Q(A)<Q(A^*)$ would contradict the definition of $A^*$.
\end{proof}

This first result will bring some understanding on a key difference between the approach of IPR and the one of coverage: the former being linked to the erroneous formula $Q(\supp(P))$ while the latter relates more to the correct one (as will be seen in the next section).
Yet, there are a few caveats that apply even to the correct version of $\alpha_\infty$. 
First, it is important to realize that this metric is impacted by the tails of $P$ and $Q$ even if they decrease very fast. An illuminating example is the following $P=\mathcal{N}(0,1)$ and $Q=\mathcal{N}(\mu,1)$. Whatever the value of $\mu$ (be it extremely large), it remains that $P$ and $Q$ have full co-support and therefore $\alpha_\infty(P,Q)=1$. 
The first negative impact of this observation is that $\alpha_\infty$ and $\beta_0$ provide a very weak characterization of the relation between $P$ and $Q$.
The second negative impact concerns the estimation of $\alpha_\infty$: namely, the tails of $P$ and $Q$ are elusive based on empirical samples, making this extreme precision the most challenging to evaluate.
Both of these observations have gone unnoticed by the previous approaches starting from \cite{kynkaanniemi_ImprovedPrecisionRecall_2019} as they purposely focused on the extreme values.
In addition to those two arguments, we would like to highlight that estimating the mass of a support is not a standard topic in machine learning. As a result, taking a binary classification standpoint as in \cite{simon_RevisitingPrecisionRecall_2019} will bring much more useful hindsight to design estimators correctly.

\subsection{Improved PR metric and follow-up works}

In this section, we describe a few published metrics related to extreme precision and recall.\footnote{We will focus on the estimate of $\alpha_\infty(P,Q)$ because swapping the role of $P$ and $Q$ entails the extreme recall $\beta_0$.}
Henceforth, we assume that one disposes of a finite set $\mathcal X$ of examples from $P$ and others in $\mathcal Y$ sampled from $Q$.

\paragraph{IPR} Proposed in \cite{kynkaanniemi_ImprovedPrecisionRecall_2019}, the Improved Precision Recall metric is given by
\begin{equation}
    \label{eq:improved}
    \hat\alpha_\infty^{iPR} := \frac{1}{\#\mathcal Y}\sum_{y\in \mathcal{Y} } \1_{\exists x\in \mathcal X, y\in \BKNN^{\mathcal X}(x)}
\end{equation}
where $\mathcal X$ and $\mathcal Y$ are the observed samples from $P$ and $Q$ respectively, 
and $\BKNN^{\mathcal X}(x)$ represents the kNN ball around $x$ computed within the set $\mathcal X$.
This value can be interpreted as the empirical estimate of $Q(\supp(P))$ where samples from $\mathcal Y$ are used to estimate the $Q$ probability, and those from $\mathcal X$ are used to estimate the support of $P$ as the union of kNN balls. 

\paragraph{Coverage} First proposed as an estimate of $\beta_0$ in \cite{naeem_ReliableFidelityDiversity_2020}, it was also adapted to $\alpha_\infty$ in \cite{khayatkhoei_EmergentAsymmetryPrecision_2023}.
\begin{equation}
    \label{eq:coverage}
    \hat\alpha_\infty^{cov} := \frac{1}{\#\mathcal Y}\sum_{y\in \mathcal{Y} } \1_{\exists x\in \mathcal X, x\in \BKNN^{\mathcal Y}(y)}
\end{equation}
Note that compared to Eq.~\eqref{eq:improved}, the condition $y\in \BKNN^{\mathcal X}(x)$ is merely replaced by $x\in \BKNN^{\mathcal Y}(y)$. 
The samples $x$ are naturally in regions of positive $P$-mass,   so that this estimate has an interpretation in terms of $Q(\cosupp(P,Q))$ : samples from $\mathcal Y$ are again used to estimate empirical probability w.r.t $Q$ but they are also used to estimate the support of $Q$ (again as the union of kNN balls associated to $\mathcal Y$) and samples from $\mathcal X$ are obviously within the support of $P$.

\paragraph{EAS} \cite{khayatkhoei_EmergentAsymmetryPrecision_2023} propose to combine both previous estimates by taking their minimum
\begin{equation}
    \label{eq:eas}
    \hat\alpha_\infty^{eas} := \min(\hat\alpha_\infty^{iPR}, \hat\alpha_\infty^{cov})
\end{equation}

\paragraph{PRC}
Proposed in \cite{cheema_PrecisionRecallCover_2023} Precision Recall Cover is an extension of coverage:

\begin{equation}
    \label{eq:prc}
\hat\alpha_\infty^{PRC} = \frac{1}{\#\mathcal{Y}}\sum_{y\in \mathcal{Y} } \1_{\#\{x\in\mathcal X/ x\in B_{kNN}^\mathcal{Y}(y)\} \geq k'}
\end{equation}
where $k'\in\N^*$ is an additional hyper-parameter: setting $k'=1$ makes this estimator identical to $cov$.

\paragraph{PPR} Last, Probabilistic Precision Recall was proposed in \cite{park_ProbabilisticPrecisionRecall_2023}:
\begin{equation}
    \label{eq:ppr}
    \hat\alpha_\infty^{PPR} := \frac{1}{\#\mathcal Y}\sum_{y\in \mathcal{Y} } \left( 1-\prod_{x\in\mathcal X} \tau(\Vert y-x \Vert) \right)
\end{equation}
where $\tau(d)=\max(0,1-\tfrac dR)$ is a fixed bandwidth tent kernel.

%% file: content/3_extensions.tex
\section{Re-interpretation and improvements}

In this section, we start by re-interpreting the $iPR$ and $cov$ estimators in terms of the dual theoretical expression in \eqref{eq:dualPR}, and then build upon this new insight to both extend the PR estimates as entire trade-off curves, as well as propose new variants.
Note that, in order to give a more complete picture of the state-of-the-art, we have presented three alternative metrics, namely $EAS$, $PRC$ and $PPR$.
Note however, that they all work in a similar fashion to IPR and coverage.
Therefore, for the sake of clarity, we will only deal with either $iPR$ or $cov$, for the remainder of the paper.

\subsection{Classification interpretation} One may notice that every estimators mentioned in the previous section reads as 
\begin{equation*}
    \hat\alpha^{M}_{\infty}=\widehat\fnr(f^{M}_\infty)
\end{equation*}
where $\widehat\fnr$ is the empirical FNR, $M$ is a reference to the metric approach (e.g. $iPR$) and $f^M_\infty$ is a classifier specific to the approach. 
In particular, one has
$f^{iPR}_\infty(z)=\1_{\#\{x\in \mathcal X / y\in\BKNN(x)\}\geq 1}$ and $f^{cov}_\infty(z)=\1_{\#\{x\in \mathcal X / x\in\BKNN(y)\}\geq 1}$.

In both cases, by design $\widehat\fpr(f^{M}_\infty)=0$ because the classifier is equal to $1$ on training samples from $\mathcal X$.
This is reminiscent of the form of $\alpha_\lambda$ in Eq.~2 when $\lambda\to\infty$:
\begin{equation*}
    \alpha_\infty = \min_{f\in\F\text{ s.t. } \fpr(f)=0} \fnr(f)
\end{equation*}
In fact, one can analyze each approach $M$, as a mere empirical version of that equation, under a restricted hypothesis class, namely:
\begin{equation*}
    \F^{M} = \{f_\gamma^M / \gamma \in [0,+\infty]\}
\end{equation*}

Note that obviously, the hypothesis class cannot be unequivocally determined by $f_\infty^M$. Yet, we shall see that natural families emerge for both $iPR$ and $cov$. Let us describe them now.
\paragraph{IPR} In that case, we set 
$$f^{iPR}_\gamma(z)=\1_{\gamma\#\{x\in \mathcal X / z\in\BKNN^{\mathcal X}(x)\}\geq \#\{y\in \mathcal Y / z\in\BKNN^{\mathcal Y}(y)\}}$$
Note that this classifier is a Kernel Density Estimator (KDE) of the form 
$$f^{iPR}_\gamma(z)=\1_{\frac{\hat p(z)}{\hat q(z)}\geq \frac 1\gamma}$$
where $\hat p(z) \propto \sum_{x\in\X} \1_{\BKNN^{\mathcal X}(x)}(z)$ and similarly for $\hat q(z)$. It is therefore a KDE classifier with adaptive bandwidth.

\paragraph{Coverage} Here we can set 
$$f^{cov}_\gamma(z)=\1_{\gamma\#\{x\in \mathcal X / x\in\BKNN^{\mathcal Y}(z)\}\geq \#\{y\in \mathcal Y / y\in\BKNN^{\mathcal X}(z)\}}$$

This resembles to a classical kNN classifier up to a minor difference: it is more standard to use the same kNN structure for both classes, that is $\BKNN^{\mathcal X\cup \mathcal Y}$ rather than using a separate one per class.
Interestingly, one can verify that the condition $\exists x\in \mathcal X\text{ s.t. } x\in\BKNN^{\mathcal Y}(y)$ is in fact equivalent\footnote{One direction is trivial since $\BKNN^{\mathcal X\cup\mathcal Y}(y) \subset \BKNN^{\mathcal Y}(y)$, the other requires a bit more 
reasoning: assuming $\exists x\in\BKNN^{\mathcal Y}(y)$ one may consider in particular the $x$ closest to $y$ and conclude that it belongs to $\BKNN^{\mathcal X\cup\mathcal Y}(y)$.} to $\exists x\in  \mathcal X\text{ s.t. } x\in\BKNN^{\mathcal X\cup\mathcal Y}(y)$. Therefore, we may as well choose
$$f^{knn}_\gamma(z)=\1_{\gamma\#\{x\in \mathcal X / x\in\BKNN^{\mathcal X\cup\mathcal Y}(z)\}\geq \#\{y\in \mathcal Y / y\in\BKNN^{\mathcal X\cup\mathcal Y}(z)\}}$$
to build the hypothesis class of $cov$. 

\paragraph{Note on symmetry}  For symmetry reasons between precision and recall, we use the previous definitions of $f^M_\gamma$  for $\gamma\geq 1$  and favor a strict inequality over a loose one for $\gamma<1$. 

\subsection{Extension and improvements}
At this stage, it is quite easy to extend the extreme precision and recall estimators into entire curves. It suffices to use Empirical Risk Minimizer approach over the hypothesis class $\F^M$ for the risk arising in Eq.~\ref{eq:dualPR}. That gives us:
\begin{equation}
    \label{eq:hatPRC}
    \hat\alpha^M_\lambda = \min_{\gamma} \lambda\widehat\fpr(f^M_\gamma) +\widehat\fnr(f^M_\gamma) 
\end{equation}

\paragraph{Splitting} The first improvement that calls upon us is to merely split the samples in two: one part used for fitting the classifier $\mathcal X^T\cup\mathcal Y^T$ and one part used for evaluation $\mathcal X^V\cup\mathcal Y^V$.
In that case, the law of large numbers applies which is crucial for the consistency of $\hat \alpha_\lambda^M$. 

Note however that because of splitting, it is possible that none of the classifier $f^M_\gamma$ ensures a null FPR. As a result, it is possible that $\hat\alpha^M_\lambda>1$. 
As a remedy, we always complement $\F^M$ with the trivial classifiers $f\equiv 1$ and $f\equiv 0$ that predict either $P$ or $Q$ uniformly.

\paragraph{Hyper-parameter $k$}
In addition to the introduction of splitting, we consider modifying either the hyper-parameter $k$ for each approach. Concerning $k$, in the kNN literature (see e.g. \cite{devroye2013probabilistic}), it is known that as the number of samples $n$ gets bigger,  $k$ can also increase but at smaller rate (this will be a key element to ensure the consistency of the kNN estimator in Thm~\ref{thm:optimality}). We therefore consider for each approach, setting $k=
\sqrt n$ in place of $k=4$ similarly to previous works. 

\paragraph{Bandwidth}
Besides, considering $iPR$, we have seen that it corresponds to an adaptive bandwidth Kernel Density Estimator (with a constant kernel). 
This design choice results in having bigger bandwidth around samples located at low density regions. 
It is responsible for the high sensibility to outliers that was pointed out in several followup works. 
We therefore consider as a simple alternative, a fixed bandwidth KDE that we will refer to as a Parzen classifier.
In that case,
$$f_\gamma^{parzen}(z) = \1_{\frac{\hat p(z)}{\hat q(z)}\geq \frac 1\gamma}$$ with
$\hat p(z) \propto \sum_{x\in\X} \1_{\Vert x-z\Vert \leq \rho_{\mathcal X})}$ and similarly for $\hat q$. In comparison to $iPR$ the bandwidth $\rho_{\mathcal X}$ is computed as the average $knn$ radius\footnote{This choice for the bandwidth is inspired by \cite{park_ProbabilisticPrecisionRecall_2023} although the Parzen variant differs from their $PPR$ estimator which does not resemble any standard classifier.} over the dataset $\mathcal X$ (and similarly for $\rho_{\mathcal Y}$) instead of using a specific bandwidth per sample. 

\subsection{Consistency analysis}

The iPR approach is known to be biased in general, since when $P=Q$, it does lead to an estimate of $\alpha_\infty$ that can be far from the true value of $1$. Indeed in such case, even when considering $n=\#\mathcal X = \#\mathcal Y\to\infty$,  $\lim_{n\to\infty} \E[\hat\alpha_\infty^{iPR}]$ can be much smaller than $1$ (see the Gaussian case in \cite{naeem_ReliableFidelityDiversity_2020}). 
On the contrary, \citet{naeem_ReliableFidelityDiversity_2020} shows that when $P=Q$,  coverage is consistent. By symmetry so is $\hat\alpha_\infty^{cov}$.

\input{content/4_fig_shifts}

We extend the above-mentioned consistency result to the entire PR curve associated to our kNN approach and in the general case of two distributions $P$ and $Q$.

 \begin{theorem}
 \label{thm:optimality}
    Let $\lambda\in\bar\R^+$, $k\geq 3$ and $n=\#\mathcal X=\#\mathcal Y$. Letting $k\to\infty$ and $\tfrac{k}{n}\to 0$, and denoting
    $$\Gamma^*_\lambda=\argmin_{\gamma} \lim_{k\to\infty, \tfrac{k}{n}\to 0} \E[\lambda \fpr(f_\lambda^{kNN})+\fnr(f_\lambda^{kNN})]$$
    Then
    \begin{enumerate}
        \item $\lambda\in\Gamma^*_\lambda$
        \item $\E[\hat\alpha_\lambda^{kNN}] \to \alpha_\lambda$ assuming that data split was used.
    \end{enumerate}
 \end{theorem}
 \begin{proof}[Proof (sketch)]     
 The proof is provided in Appendix~\ref{app:proof} and is similar to the standard Bayes consistency results of the kNN classifier (see e.g. \cite{devroye2013probabilistic}[chap 5\& 
 6]).
     It is merely adapted to the fact that the risk is class weighted i.e. $R_\lambda(f)=\lambda \fpr(f) + \fnr(f)$ instead of the classical one $R(f)=\frac 12(\fpr(f)+\fnr(f))$.
 \end{proof}

%% file: content/4_fig_shifts.tex
\definecolor{color_IPR}{RGB}{23,116,180}
\definecolor{color_KNN}{RGB}{255,136,32}
\definecolor{color_PARZ}{RGB}{20,160,20}
\definecolor{color_COV}{RGB}{214,27,28}
\definecolor{color_GT}{RGB}{110,24,189}
\newlength{\myheight} 
\setlength{\myheight}{0.23\textwidth}
\begin{figure*}[t]
    \centering
    \begin{tabular}{cccc}
        \multicolumn{2}{c}{\emph{50\% split validation/train}} 
        &
        \multicolumn{2}{c}{\emph{without split}}  
        \\
        \includegraphics[height=\myheight,trim={0 0 160 0},clip]{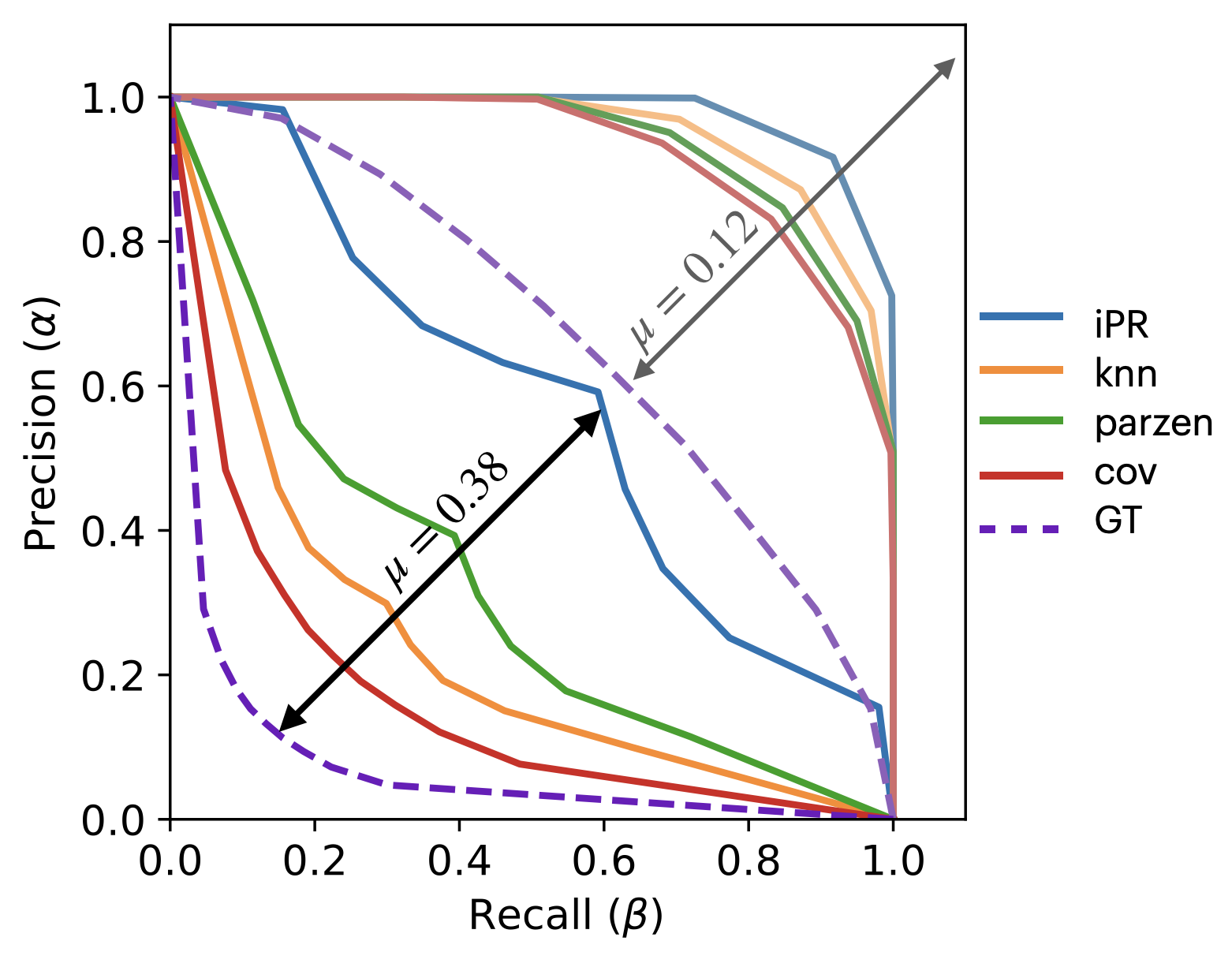}
        &
        \includegraphics[height=\myheight,trim={20 0 160 0},clip]{
        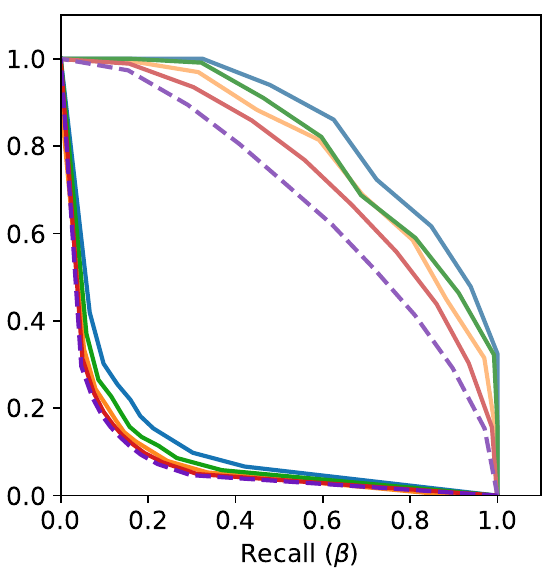}
        &
        \includegraphics[height=\myheight,trim={20 0 160 0},clip]{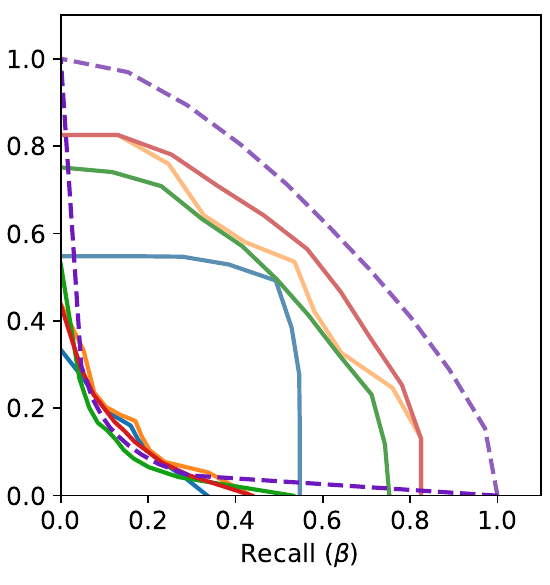}
        &
        \includegraphics[height=\myheight,trim={20 0 160 0},clip]{
        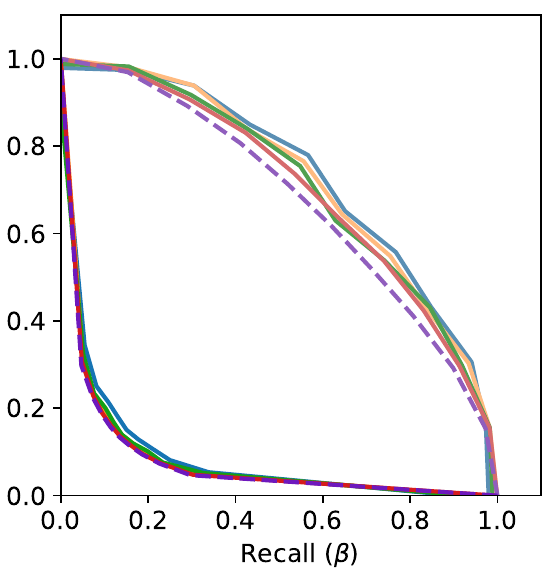}
         \\
        $k=4$ & $k = \sqrt n$ & 
        $k=4$ & $k = \sqrt n$ 
    \end{tabular}

    \caption{\textbf{Comparing two shifted Gaussians.} 
    The Ground-Truth PR curve (
    {\color{color_GT} \bfseries -~-\textsc{GT}})
    is compared to empirical estimates from various NN-classifiers:
    {\color{color_IPR} \bfseries --\textsc{iPR}},
    {\color{color_KNN} \bfseries --\textsc{knn}},
    {\color{color_PARZ} \bfseries --\textsc{Parzen}}, and 
    {\color{color_COV} \bfseries --\textsc{Coverage}}.
    Here $P \sim \mathcal{N}(0,\mathbb{I}_{d})$ and $ Q \sim \mathcal{N}(\mu \mathbf{1}_{d},\mathbb{I}_{d})$ with $d=64$ dimensions and $\mu=\frac{1}{\sqrt d}\approx.12$ or $\mu=\frac{3}{\sqrt d}\approx.38$. 
    $n=10$K points are sampled using $k=4$ or $k=\sqrt n$ for NN comparison, with or without dataset validation/train split.
    (Curves are averaged over 10 random samples, see Appendix).
    }
    \label{fig:shift-gauss}
\end{figure*}

%% file: content/4_experiments.tex
\section{Experiments}

In this section, we reissue a few experiments on toy datasets that were proposed in the literature. 
In all settings, for the different estimators under scrutiny, we use $n=10$K samples.
In each experiment, the distributions $P$ and $Q$ are known analytically and the ground-truth PR curve can be estimated easily because the Bayes classifier is the likelihood ratio classifier \cite{simon_RevisitingPrecisionRecall_2019} $f_\lambda^*(z)=\1_{\frac{dQ}{dP}(z)\leq\lambda}$. To obtain high accuracy ground-truth curves we resort to a large sample ($n^{GT}=100$K) and estimate $$\hat\alpha^{GT}_\lambda=\lambda \widehat \fpr(f_\lambda^*)+ \widehat\fnr(f_\lambda^*)$$

Based on this PR curve, we can either evaluate the quality of an estimator visually, or use a scalar indicator to summarize the quality of the estimator. 
In particular, we propose to use the Jaccard index (a.k.a IoU) between the ground-truth curve and the estimator under review. 
This index is always smaller than $1$ and the larger the better.

\subsection{Gaussian shifts}\label{sec:exp_shift}

Inspired by recent experiments made in the literature, 
we consider the case where $P$ and $Q$ are two Gaussian distributions with an increasing shift.
In order to match the published experiments, we generate $10$K data points from each Gaussian distribution, in $\mathbb{R}^{64}$.
Even though the number of sampled points is quite high, we decide to run the experiments $100$ times to have robust interpretations of the metrics.
What interests us here is the behavior of the of the estimated curves in comparison to the ground-truth. 
In particular,  unlike \cite{kynkaanniemi_ImprovedPrecisionRecall_2019,naeem_ReliableFidelityDiversity_2020,park_ProbabilisticPrecisionRecall_2023} we reaffirm that the extreme values of the curves should be equal to $1$ as the supports of the two Gaussians are always the same. 

We run the experiments on 5 different methods, using 4 different shifts for the fake Gaussian. We present the resulting curves for only two shifts in Fig.~\ref{fig:shift-gauss} and condense the complete experiment results in Appendix in Table~\ref{tab:mean_iou_shift_k4}. 
In these settings, the standard deviation for each experiment is lower than $10^{-2}$. Whereas the maximum values for the PR curves are always at 1, the curves themselves show the lack of information to compare the two distributions. While methods from the literature, the number of nearest neighbors $k$ used in the manifold estimations is set to 3 (\cite{kynkaanniemi_ImprovedPrecisionRecall_2019}) and to 5 (\cite{naeem_ReliableFidelityDiversity_2020}), we first set $k=4$ then, motivated by \ref{thm:optimality}, we set $k=\sqrt{n}$. We observe that the yielded results are more convincing with $k=\sqrt{n}$, and that the difference between the split and no split scenarios is marginal. As we have theoretical guarantees in the case with split, we decided to keep the combination split and $k=\sqrt n$ for the following experiments (complementary results are presented in the appendix).

\input{content/4_gmm}

\subsection{Gaussian mixture models}
\label{sec:exp_GMM}

In \cite{luzi_EvaluatingGenerativeNetworks_2023} the authors advocate that Inception features are better approximated by Gaussian Mixtures than pure Gaussians. Besides this scenario allows to illustrate the phenomenon of mode dropping (mode present in $P$ but not in $Q$), mode creation (mode present in $Q$ but not in $P$), and mode reweighing (shared modes between $P$ and $Q$ weighted differently) as illustrated in Fig.~\ref{fig:mode-diffs}.
In our toy experiment, we sample points from two GMM  in dimension $d=64$, $n=1$K  samples, splitting is applied, and $k=\sqrt n$ (see Fig.~\ref{fig:GMM-dim64}).
The two GMM $P$ and $Q$ are set as follows:  $P \sim \sum_{\ell} p_\ell\, \mathcal{N}(\mu_\ell \mathbf 1_{d}, \mathbb{I}_{d})$  and  $ Q \sim \sum_{\ell} q_\ell\, \mathcal{N}(\mu_k \mathbf 1_{d}, \mathbb{I}_{d})$ with $d=64$ dimensions and  $\mu_\ell \in \{0, -5, 3, 5\}$. However, $P$ and $Q$ have different weights ($p_\ell$ and $q_\ell$) $p_\ell  \in \{0.3, 0.2, 0.5, 0\}$ $q_\ell  \in \{0, 0.5, 0.2, 0.3\}$. 
kNN, parzen and coverage perform very well with respect to the ground-truth, while IPR overestimates the PR-curve and especially fails at catching the re-weighting transition.

\subsection{Outliers}\label{sec:exp_outlier}
One of the main contributions of \cite{naeem_ReliableFidelityDiversity_2020} over \cite{kynkaanniemi_ImprovedPrecisionRecall_2019} was its robustness to outliers. We here investigate how outliers affect our PRD curves. 
To do so, we simply once again define the $P$ and $Q$ distributions as two shifted Gaussians. 
Similarly to experiments in literature, we then add a single outlier $x_\text{outlier}$ to the real distribution $P$ such that $x_\text{outlier}= \mathbf 4$. 
All the cases are considered in appendix, yet the observations are simple: as already reported in the aforementioned works, we observe that iPR classifier is indeed affected by such a perturbation for low k-NN comparison ($k=4$) and without split. This sensitivity is especially strong near the extreme values. We also notice that this effect is easily mitigated when using a larger set $k = \sqrt{n}$.
On the other hand, other PRD curves based on more robust classifiers (Parzen, Coverage and kNN) are not affected by the outlier as expected.

\subsection{Impact of the ambient dimension}
\label{sec:exp_dim}
It is customary in the evaluation of generative models to first feed both target and generated data through a deep classifier feature space. In the protocol for computing FID, the authors use InceptionV3 where the feature space has a dimension of 2048, and use 10K samples for both real and generated data. We now test our metrics in the same Gaussian shifting setting, with samples in $\mathbb{R}^{2048}$. The experiment result is shown in Fig.~\ref{fig:gaussian-shift-2048}. While we used $k=\sqrt{n}$, the results show a much larger overestimation than in the $\mathbb{R}^{64}$ setting. This illustrates the well-known curse of dimensionality. 

\begin{figure}
    \centering
    \includegraphics[width=.7\linewidth,trim={0 0 160 0},clip]{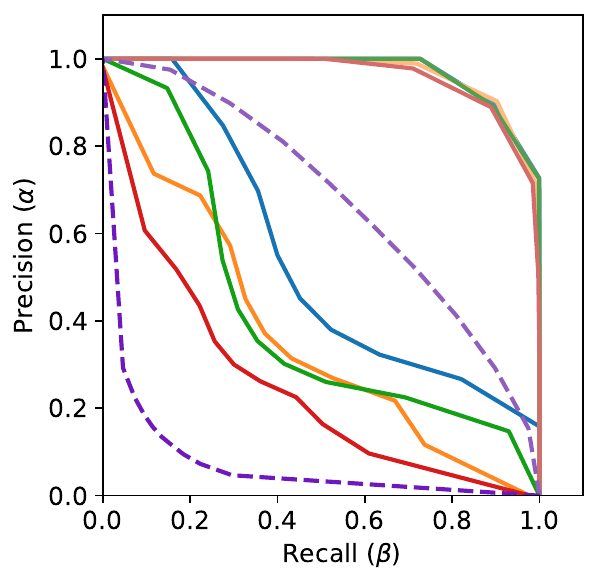}
    \caption{\textbf{PR curves in high dimension} Same experiment as in Fig.~\ref{fig:shift-gauss} (50\% split, $n=10$K, $k=\sqrt{n}$) with $d=2048$ dimensions and $\mu \in \{\frac{1}{\sqrt{d}},\frac{2}{\sqrt{d}}\}$. 
    }
    \label{fig:gaussian-shift-2048}
\end{figure}

\subsection{Variability study}\label{sec:exp_variability}
Additional experiments in appendix (Fig.~\ref{fig:varibility-nbpoints})
scrutinize the impact of the size sample $n$ on the variability of the evaluation curves.
They empirically illustrate the consistency of the proposed method based on robust classifiers when increasing (Thm~\ref{thm:optimality}). Empirically, using $10K$ points reduces sufficiently the variability to make comparison between curves reliable, 
which is in line with the standard usage for generative model evaluation \cite{heusel_GANsTrainedTwo_2017,sajjadi_AssessingGenerativeModels_2018}.

%% file: content/4_gmm.tex
\begin{figure*}[!ht]
     \centering
     \begin{subfigure}[b]{0.67\textwidth}

    \includegraphics[width=0.56\textwidth,trim={42, 0, 42, 0}, clip]{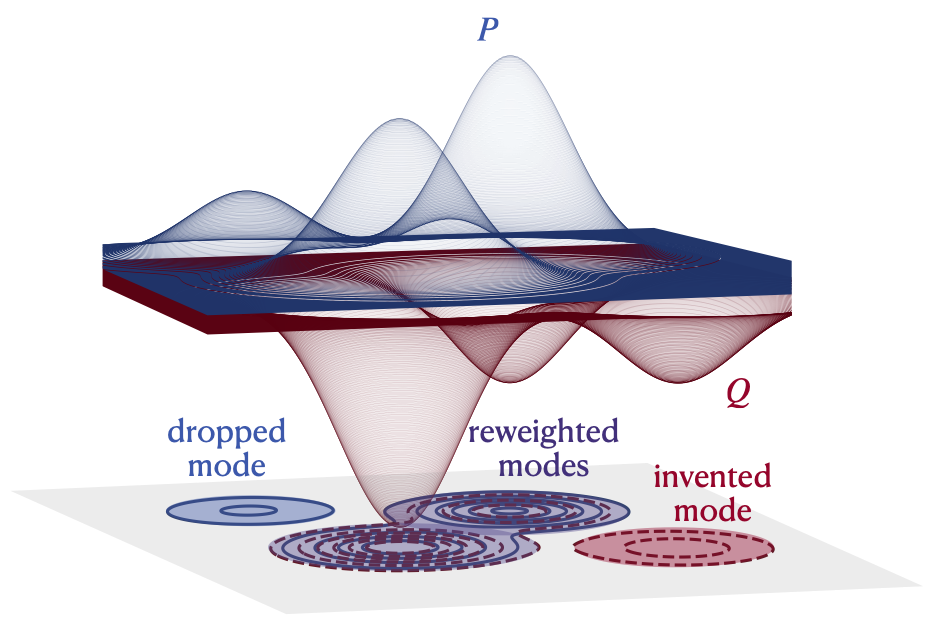}
    \includegraphics[width=0.43\textwidth,trim={8, 0, 30, 0}, clip]{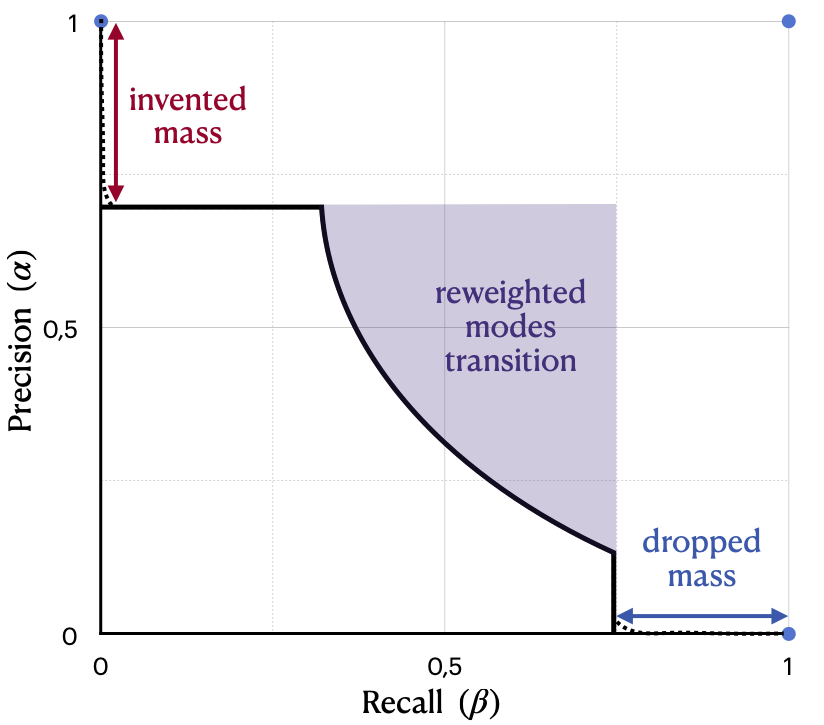}
    \caption{Left: Gaussian Mixture Models $P$ and $Q$ showing mode dropping (only in $P$), mode inventing (only $Q$), and mode re-weighting (in both but distributed differently). Right:~Expected coarse shape of the PR-curve (solid black). Note that due to the infinite tails of the Gaussian modes, the vertical and horizontal transitions are theoretically smooth and reach $1$ (dashed curve).
    }
    \label{fig:mode-diffs}
\end{subfigure}
    \hfill
\begin{subfigure}[b]{0.3\textwidth}
    \includegraphics[width=0.9\textwidth,trim={0 0 135 0},clip]{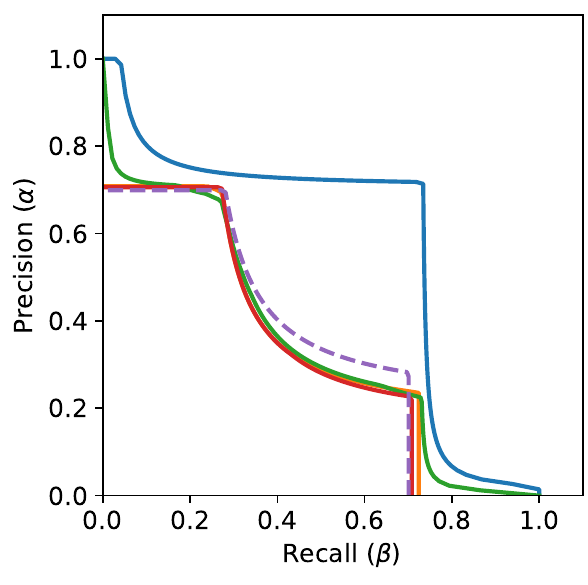}
    \caption{\textbf{Comparing two Gaussian mixtures}.~\ref{fig:mode-diffs}.
    The Ground-Truth PR curve (
    {\color{color_GT} \bfseries -~-\textsc{GT}})
    is compared to empirical estimates from various NN-classifiers:
    {\color{color_IPR} \bfseries --\textsc{iPR}},
    {\color{color_KNN} \bfseries --\textsc{knn}},
    {\color{color_PARZ} \bfseries --\textsc{Parzen}}, and 
    {\color{color_COV} \bfseries --\textsc{Coverage}}.
    %
    %
    }
    \label{fig:GMM-dim64}
\end{subfigure}
\end{figure*}

%% file: content/5_discussion.tex
\section{Discussion and perspectives} 
\subsection{Distilling the curve with two scalar metrics}
Practitioners may enjoy summarizing the PR curve with two metrics reflecting respectively precision and recall.
In other words, one would be willing to trade exhaustiveness for conciseness.
This may be particularly useful to ease comparison between models. 
In that prospect, we have seen that extreme precision and recall are yet far from ideal.
We therefore discuss two alternatives and comment on them in a scenario combining mode dropping/invention/re-weighting (e.g. Fig.~\ref{fig:mode-diffs}).

\paragraph{F-scores} Proposed by \cite{sajjadi_AssessingGenerativeModels_2018}, the $F_b$ score is defined as:
$$F_b = \max_{\lambda \in [0,+\infty]} \frac{1+b^2}{\frac{b^2}{\alpha_\lambda}+\frac 1{\beta_\lambda}}
.$$
When $b\to\infty$, $F_b\nearrow\alpha_\infty$ and respectively when $b\to 0$ $F_b\searrow\beta_0$, so that \cite{sajjadi_AssessingGenerativeModels_2018} proposed to consider $F_b$ and $F_{1/b}$ with a large value of $b$ (namely $b=8$).
Although this aspect was not discussed in their original work, one can understand that these metrics 
will not be sensitive to rapidly decaying infinite tails.
However, observe that the score $F_b$ is a weighted harmonic mean that is computed from a single (optimal) point of the PR curve, which can vary dramatically without affecting the metric.
As a consequence, when $b$ is large, $F_b$ and $F_{1/b}$ will mainly capture pure mode dropping/invention and remain superficially indicative of mode re-weighting.

\paragraph{PR median} As another alternative, one may consider $(\alpha_{\bar \lambda},\beta_{\bar \lambda})$ where $\bar\lambda$ is set so that the line $\alpha=\bar\lambda\beta$ cuts the region under the PR curve into two sub-parts of equal areas. 
In the case of pure mode dropping and mode inventions, $\alpha_{\bar\lambda}=\alpha_\infty$ and $\beta_{\bar\lambda}=\beta_0$.
On the contrary, when infinite tails create quick transitions to $\alpha_\infty=1$ like in Fig.~\ref{fig:mode-diffs}, then (alike $F_b$) $\alpha_{\bar\lambda}<1$ will be but mildly impacted by rapid-decay tails. 
As opposed to $F_b$, the value of $\alpha_{\bar\lambda}$ will be largely affected by the presence of a transition due to mode re-weighting (violet transition in Fig.~\ref{fig:mode-diffs}).
As a result, these metrics may be preferred when one would like to account for mode re-weighting in addition to mode dropping/invention.

An empirical assessment on the pros and cons of both alternatives would be a valuable endeavor. In particular,
one could study the behavior of the said metrics with respect to hyperparameters of state-of-the-art generative models.
One can for instance consider, the truncation procedure for GANs, or the guidance scale factor for diffusion models.
This empirical study is left as future work.

\subsection{More in-depth convergence analysis}
In Theorem~\ref{thm:optimality}, we have demonstrated that the kNN estimator is universally consistent. 
The proof is adapted from standard Bayes consistency results on kNN classifiers. 
Characterizing the quality of other estimators based on standard classification schemes (e.g. Kernel Discriminant Analysis) could also be considered. Besides, characterizing, rates of convergence under regularity assumptions for $P$ and $Q$ is also an appealing avenue.
This is left as future work, but we refer the reader to \cite{devroye2013probabilistic} and to \cite{gyorfi2021universal} for a recent overviews of useful results.
In a more practical perspective, one may wish to choose optimally the hyper-parameters of the different estimators. 
In our work we have considered the following heuristics $k=\sqrt{n}$ for the kNN estimator and a fixed bandwidth $\rho_{\mathcal{X}}$ computed based on the average kNN radius for the Parzen classifier. 
A large literature does exist around those topics, see for example \cite{ghosh2004optimal} and \cite{doring2018rate} for kNN and \cite{ghosh2006optimum} for optimal bandwidth.
It is also possible to resort to cross-validation for setting these hyper-parameters.

%% file: content/6_conclusion.tex
\section{Conclusion}

In this work, we have given a new perspective on recent metrics used to evaluate extreme precision and recall of generative models.
Doing so we have obtained two by-products.
First, we have presented a systematic way to extend the extreme values to obtain complete PR curves.
Second, we built upon standard literature in non-parametric classification to improve the original approaches.
In particular, we have provided a consistency result for the kNN PR-curve variant as well as several practical improvements over the original iPR and coverage metrics.
We have also studied the empirical behavior of the obtained variants in the light of several toy datasets experiments. 

Our main messages are the following. First, computing non extreme PR values is crucial because of essential issues in the extreme values which are related to their sensitivity to the distribution tails.
Then, the curves themselves allow to describe more finely how the masses of the two distributions under comparison differ on their modes. This is useful in practice in order to tackle the case where a model generates data from the target support but with re-weighted masses.
On the experimental side, it emerges that coverage is indeed better suited than iPR. Both approaches can be improved by adapting the number of neighbors $k$ with respect to the number of available samples. 
If employing a data split is theoretically appealing, its empirical impact is less marked since the negative bias resulting from the lack of split can sometimes advantageously compensate the positive bias caused by the restricted hypothesis class. However, this benefit is not consistent over all experiments.

%% file: content/9_appendix.tex
\appendix

\section{Proof of Theorem~\ref{thm:optimality}: consistency}
\label{app:proof}
\begin{proof}
     To establish the proof, we need only show that $R_\lambda(f^{kNN}_{\lambda})\to \alpha_\lambda$ as $k\to\infty$ and $\tfrac{k}{n}\to 0$.
     This will effectively imply both items in the theorem since $\alpha_\lambda$ is the associated Bayes risk \cite{simon_RevisitingPrecisionRecall_2019}. 
     To establish this limit, the first step is to show that for fixed $k$ then $\lim_{n\to\infty} R_\lambda(f^{kNN}_{\gamma})$ is equal to
     \begin{equation}
     \label{eq:rlambda}
         \begin{split}
     2\lambda \E[\eta(Z)\P\{Binom(k,\eta(Z))<\tfrac{k}{\gamma+1}|Z\}] +\\
     2\E[(1-\eta(Z))\P\{Binom(k,\eta(Z))>\tfrac{k}{\gamma+1}|Z\}] 
         \end{split}
     \end{equation}
     where $Z=UX+(1-U)Y$ with $X\sim P$, $Y\sim Q$  and $U$ is a fair coin random variable so that $Z\sim \tfrac{P+Q}{2}$ and $\eta(Z):=\P(U=1|Z) = \tfrac{dP}{d(P+Q)}(Z)$ (respectively $1-\eta(Z):=\P(U=0|Z) = \tfrac{dQ}{d(P+Q)}(Z)$). The demonstration of Eq.~\eqref{eq:rlambda} follows the same argument as in \cite{devroye2013probabilistic}[Thm.~5.2] (up to the occurrence of $\lambda$ and $\gamma$ weights) and is not repeated here for the sake of conciseness.

    Now, taking $\gamma=\lambda$ we want to show that the previous expression tends to $\alpha_\lambda$ which for the recall equals $(\lambda P\wedge Q)(\measSpace)$ or expressed otherwise as $2\E[\lambda \eta(Z)\wedge(1-\eta(Z))]$.

     Eq.~\eqref{eq:rlambda} can be reformulated as $\lim_{n\to\infty} R_\lambda(f^{kNN}_{\lambda}) = 2\E[\mu_\lambda(\eta(Z))]$ with
     \begin{equation}
         \label{eq:mu}
         \begin{split}
         \mu_\lambda(p) =& \lambda p \P\{Binom(k,p) < \tfrac{k}{\lambda+1}\} \\
                 & \hspace{1em} + (1-p) \P\{Binom(k,p>\tfrac{k}{\lambda+1}\}\\
         \end{split}
     \end{equation}
     So that it suffices to show that $\forall p\in[0,1]$, $\mu_\lambda(p)\to\lambda p \wedge(1-p)$.
     
     Let's proceed by cases, starting by considering $\lambda p < (1-p)$ which is also equivalent to $p<\tfrac{1}{\lambda+1}$. In that case we need to show that $2\mu_\lambda(p)\to \lambda p$.
     Denoting $q_\lambda(p)  = \P\{Binom(k,p)>\tfrac{k}{\lambda+1}\}$, we have
     \begin{equation}
     \begin{split}
        \mu_\lambda(p) =&  \lambda p (1-q_\lambda(p)) + (1-p)q_\lambda(p) \\
                =& \lambda p + q_\lambda(p)(1-(\lambda+1)p)
     \end{split}
     \end{equation}
     Using Hoeffding's inequality (i.e. $\forall t>0, \P\{Binom(k,p)-kp> t\}\leq \exp (-2k t^2)$  and obtain
     \begin{equation}
         \begin{split}
     q_\lambda(p) =& \P\{Binom(k,p)-kp> k (\tfrac{1}{\lambda+1}-p)\} \\
     \leq& \exp \left(-2k (\tfrac{1}{\lambda+1}-p)^{2}\right)
         \end{split}
     \end{equation}
     Note that the assumption $p < \tfrac 1{\lambda+1}$ is crucial to apply Hoeffding's inequality (because $t$ needs to be positive).
     The right hand side converges to $0$ as $k\to\infty$ because by assumption $p\neq \tfrac{1}{\lambda+1}$.

     The case where $\lambda p > (1-p)$ (or $p>\tfrac{1}{\lambda+1}$) is similar and is left to the reader. In that case, we obtain $\mu_\lambda(p)\to (1-p)$.
     There remains the case of equality, that is $\lambda p = 1-p = \tfrac{1}{\lambda +1}$. In that case, even without taking the limit, one can check that $\mu_\lambda (p) = \lambda p$, which concludes the proof.
\end{proof}

\section{Additional experimental results}


\textbf{Computation of PR curves}
The computation PR curves involves the computation of false positive rate (fpr) and false negative rate (fnr) for various classifiers $f_\gamma^M$ parameterized by $\gamma$.
In experiments, we consider random samples $\mathcal X$ and $\mathcal Y $ from distributions $P$ and $Q$ with $n=|\mathcal X | = |\mathcal Y |$.
$$
    \widehat \fnr (f_\gamma) = \frac1n 
    \sum_{y \in \mathcal Y} {f_\gamma(y)} 
$$
and
$$
    \widehat \fpr(f_\gamma) = 
     \frac1n \sum_{x \in \mathcal X} {1 - f_\gamma(x)} 
$$
Precision $\hat \alpha_\lambda$, and recall $\hat \beta_\lambda = \tfrac{1}{\lambda}\alpha_\lambda$ are computed using 
Eq.~\eqref{eq:hatPRC},
where parameters $\lambda = \text{atan}(\theta)$ and $\gamma$ are both uniformly sampled with $\theta \in (0, \tfrac{\pi}{2})$.

\begin{figure}[!htb]
    \centering
    \begin{tabular}{cc}
        \multicolumn{2}{c}{\emph{without split}}  
        \\
        \includegraphics[height=\myheight,trim={0 0 0 0},clip]{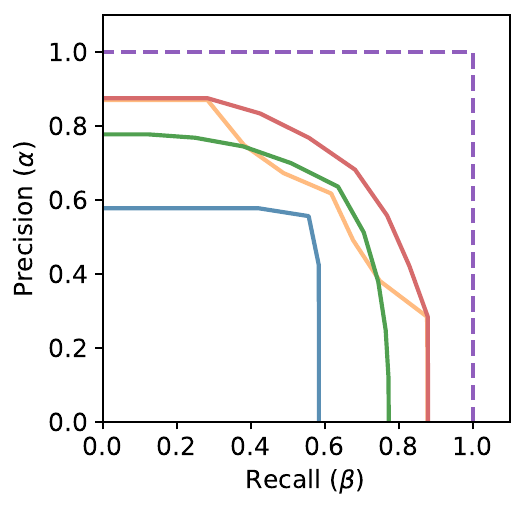}
        &
        \includegraphics[height=\myheight,trim={0 0 0 0},clip]{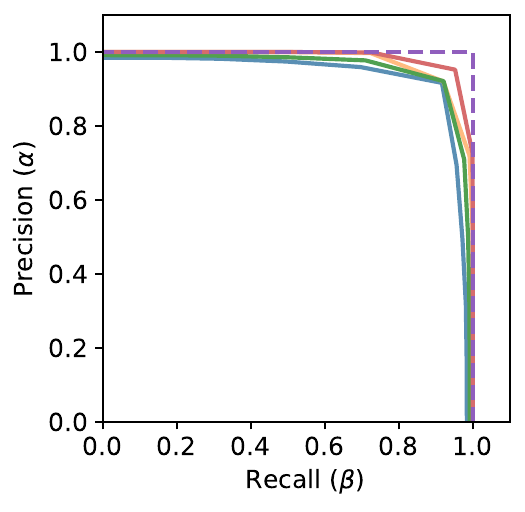}
         \\
        $k=4$ & $k=\sqrt n$
        \\
        \multicolumn{2}{c}{\emph{with split}}  
        \\
        \includegraphics[height=\myheight,trim={0 0 0 0},clip]{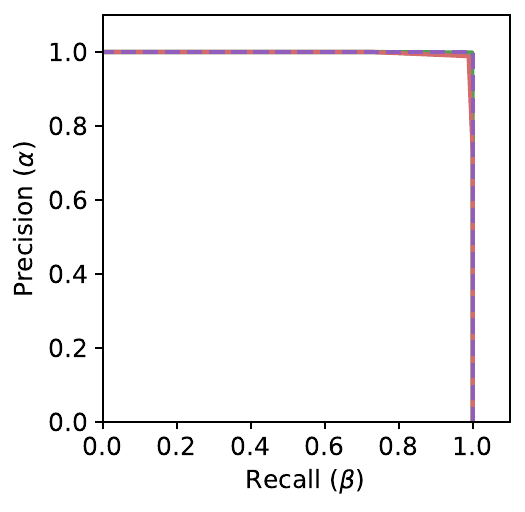}
        &
        \includegraphics[height=\myheight,trim={0 0 0 0},clip]{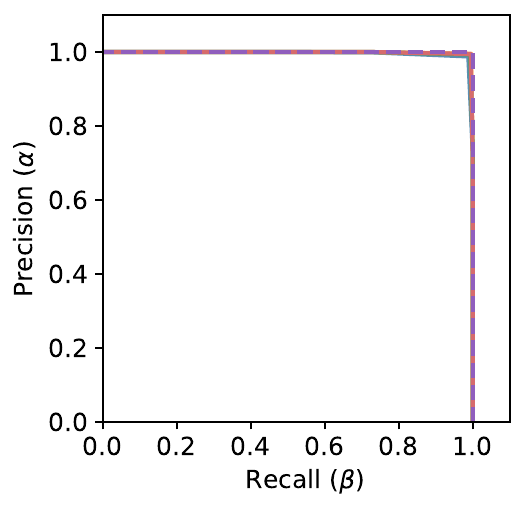}
         \\
        $k=4$ & $k=\sqrt n$
    \end{tabular}
    
    \caption{\textbf{Illustration of the impact of splitting for $P=Q$.} 
    The setting is the same as Fig.~\ref{fig:shift-gauss}
    for a translation of $\mu=0$ between two Gaussian in dimension $d=64$ (curves are averaged over 3 random samples).
    The Ground-Truth PR curve (
    {\color{color_GT} \bfseries -~-\textsc{GT}})
    is compared to empirical estimates from various NN-classifiers:
    {\color{color_IPR} \bfseries --\textsc{iPR}},
    {\color{color_KNN} \bfseries --\textsc{knn}},
    {\color{color_PARZ} \bfseries --\textsc{Parzen}}, and 
    {\color{color_COV} \bfseries --\textsc{Coverage}}.
    Top reports results without splitting :
    as reported in the literature, estimated extremal precision and recall values are not equal to 1, contrary to the ground-truth.
    Bottom curves, obtained with a 50\% splits, are very close to the ideal curve.
    }
    \label{fig:P=Q}
\end{figure}

\textbf{Splitting}
All experiments involving data splitting into training and validation sets ($\mathcal X^T \& \mathcal Y^T$ and $\mathcal X^V \& \mathcal Y^V$) are 50\%: curves computed in this setting therefore rely on $\frac n2$ data points.

Figure \ref{fig:P=Q} provides a visual illustration of the practical impact of splitting the dataset into separate training and validation sets to assess precision-recall curves. 
We consider same experimental setting as in section \ref{sec:exp_shift} and Figure \ref{fig:shift-gauss} with $\mu=0$, in such a way that $P=Q$.
In this specific case, the precision is equal to 1 for all recall.

\textbf{Gaussians shifts}
Table.~\ref{tab:mean_iou_shift_k4} complements section \ref{sec:exp_shift} and Figure \ref{fig:shift-gauss} by comparing various estimated PR curves with respect to the ground-truth, using average IoU scores.

\begin{table}[htb]
\small
    \centering
    \begin{tabular}{c c c||c|c|c|c}
&&shift $\mu$ 
   & {\color{color_IPR}  \textsc{iPR}}
   & {\color{color_KNN}  \textsc{knn}}
   & {\color{color_PARZ}  \textsc{Parzen}} 
   & {\color{color_COV}  \textsc{Coverage}}
\\
\hline\hline
\multirow{8}{*}{\rotatebox[origin=c]{90}{with 50 \% split}}&\multirow{4}{*}{\rotatebox[origin=c]{90}{$k=4$}}
&$0.12$  &  0.69  & 0.71  & 0.72  & 0.73  \\
&&0.21  &  0.42  & 0.49  & 0.49  & 0.55  \\
&&0.29  &  0.24  & 0.38  & 0.34  & 0.48  \\
&&0.38  &  0.13  & 0.33  & 0.24  & 0.48  \\
\cline{2-7}
&\multirow{4}{*}{\rotatebox[origin=c]{90}{$k=\sqrt n$}}
&0.12  &  0.81 & 0.87 & 0.84 & 0.92 \\
&&0.21  &  0.69 & 0.84 & 0.78 & 0.90 \\
&&0.29  &  0.65 & 0.84 & 0.75 & 0.90 \\
&&0.38  &  0.63 & 0.84 & 0.75 & 0.93 \\
\cline{1-7}
\multirow{8}{*}{\rotatebox[origin=c]{90}{without split}}&\multirow{4}{*}{\rotatebox[origin=c]{90}{$k=4$}}
&0.12  &  0.43 & 0.7 & 0.62 & 0.76 \\
&&0.21  &  0.55 & 0.81 & 0.68 & 0.84 \\
&&0.29  &  0.62 & 0.79 & 0.68 & 0.77 \\
&&0.38  &  0.55 & 0.61 & 0.62 & 0.63 \\
\cline{2-7}
&\multirow{4}{*}{\rotatebox[origin=c]{90}{$k=\sqrt n$}}
&0.12  &  0.91 & 0.93 & 0.94 & 0.96 \\
&&0.21  &  0.88 & 0.93 & 0.92 & 0.97 \\
&&0.29  &  0.84 & 0.92 & 0.90 & 0.95 \\
&&0.38  &  0.83 & 0.91 & 0.90 & 0.96
\end{tabular}
    \caption{\textbf{Mean IoU scores} for shifting Gaussians.
    Standard deviations are $< 10^{-2}$ with $n=10$K.}
    \label{tab:mean_iou_shift_k4}
\end{table}

\textbf{Gaussian Mixture comparison}
Fig.~\ref{fig:GMM-dim64_suite} complements section \ref{sec:exp_GMM} and Figure \ref{fig:GMM-dim64} with additional curves for different setting (w/ and w/o splitting, $k=4$ or $k=100$).

\begin{figure*}[htb]
    \centering
    \begin{tabular}{cccc}
        \multicolumn{2}{c}{\emph{50\% split validation/train}} 
        &
        \multicolumn{2}{c}{\emph{without split}}  
        \\
        \includegraphics[height=\myheight,trim={0 0 130 0},clip]{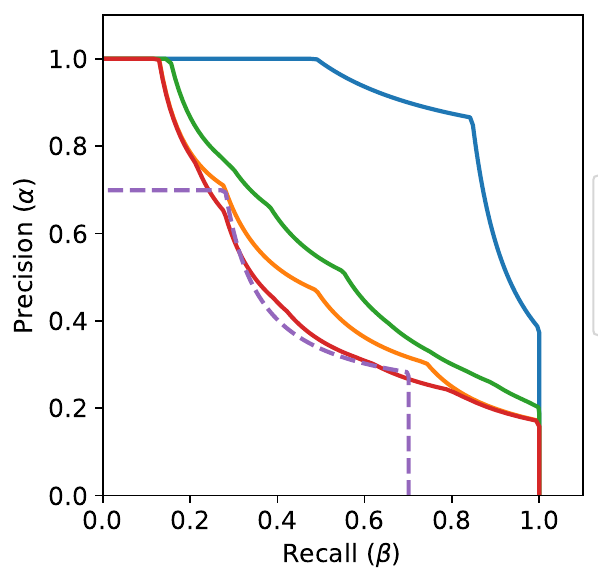}
        &
        \includegraphics[height=\myheight,trim={20 0 130 0},clip]{figures/GMM_dim64_k-sqrtn_split.pdf}
        &
        \includegraphics[height=\myheight,trim={20 0 130 0},clip]{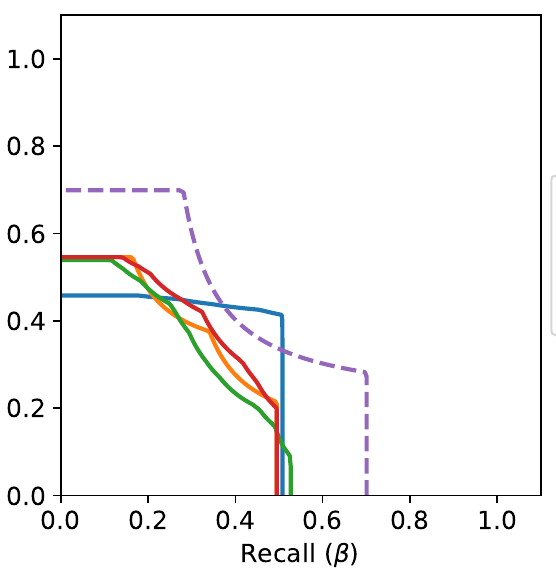}
        &
        \includegraphics[height=\myheight,trim={20 0 130 0},clip]{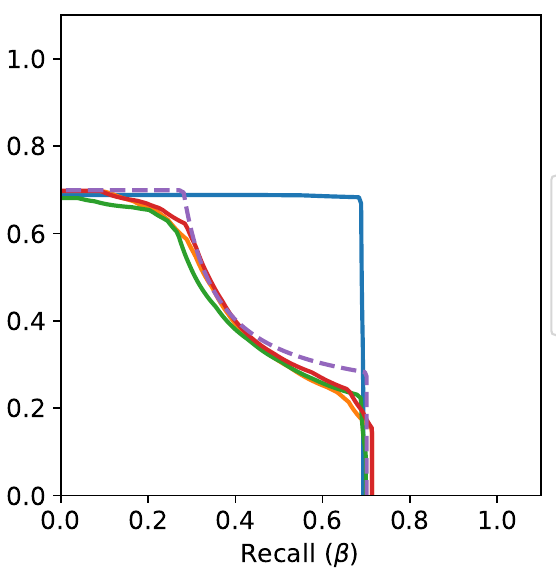}
         \\
        $k=4$ & $k = \sqrt n$ & 
        $k=4$ & $k = \sqrt n$ 
    \end{tabular}
    \caption{\textbf{Comparing two Gaussian mixtures}. This figure complements Fig.~\ref{fig:GMM-dim64}.
    The Ground-Truth PR curve (
    {\color{color_GT} \bfseries -~-\textsc{GT}})
    is compared to empirical estimates from various NN-classifiers:
    {\color{color_IPR} \bfseries --\textsc{iPR}},
    {\color{color_KNN} \bfseries --\textsc{knn}},
    {\color{color_PARZ} \bfseries --\textsc{Parzen}}, and 
    {\color{color_COV} \bfseries --\textsc{Coverage}}.
    Here $P$ and $Q$ are two GMMs sharing the same modes (centered at $\mu_k$):
    $P \sim \sum_{\ell} p_\ell \mathcal{N}(\mu_\ell \mathbf 1_{d}, \mathbb{I}_{d})$ 
    and 
    $ Q \sim \sim \sum_{\ell} q_\ell \mathcal{N}(\mu_k \mathbf 1_{d}, \mathbb{I}_{d})$ with $d=64$ dimensions and 
    $\mu_\ell \in \{0, -5, 3, 5\}$.
    However, $P$ and $Q$ have different weights ($p_\ell$ and $q_\ell$)
    $p_\ell  \in \{0.3, 0.2, 0.5, 0\}$
    $q_\ell  \in \{0, 0.5, 0.2, 0.3\}$.
    $n=1$k points are sampled and split in half between validation and train, and $k=\sqrt n$. 
    %
    }
    \label{fig:GMM-dim64_suite}
\end{figure*}

\textbf{Outlier}
Fig.~\ref{fig:outlier-gauss} complements section \ref{sec:exp_outlier}.
It shows the impact of having a single outlier (a data-point out of $n=10$K) in one of the samples. 
Here only the sample $\mathcal X$ from $P$ is polluted, thus mainly affecting precision.

\begin{figure}[htb]
    \centering
    \begin{tabular}{cc}
        \multicolumn{2}{c}{$k = 4$ \emph{without split}}  
        \\
        \includegraphics[height=\myheight,trim={20 0 135 0},clip]{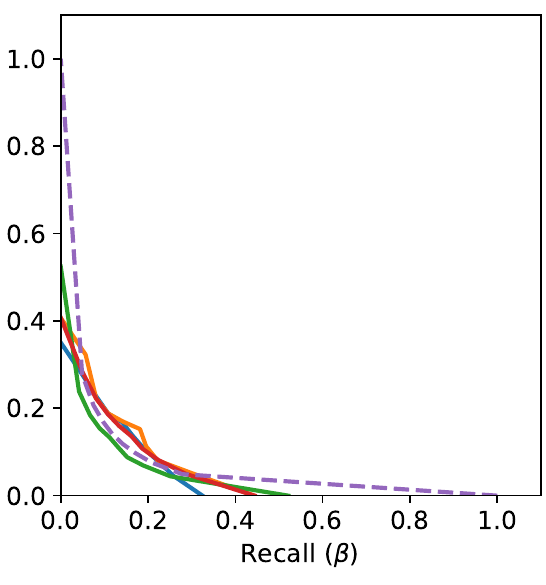}
        &
        \includegraphics[height=\myheight,trim={0 0 135 0},clip]{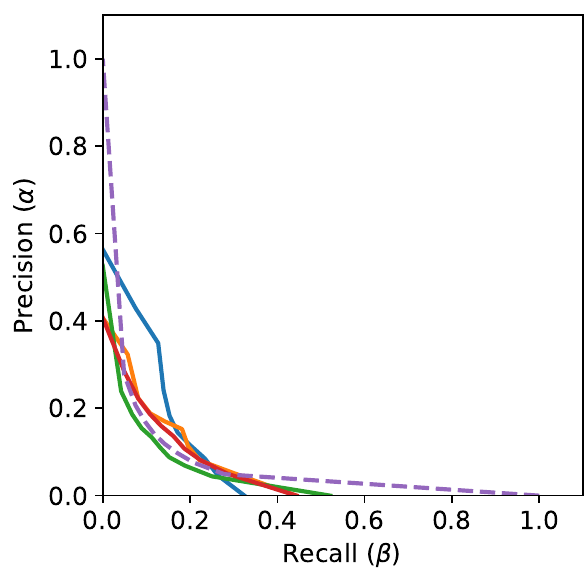}
         \\
        w/o outlier & w/ outlier
        \\
        \multicolumn{2}{c}{$k = \sqrt n$ \emph{without split}}  
        \\
        \includegraphics[height=\myheight,trim={20 0 135 0},clip]{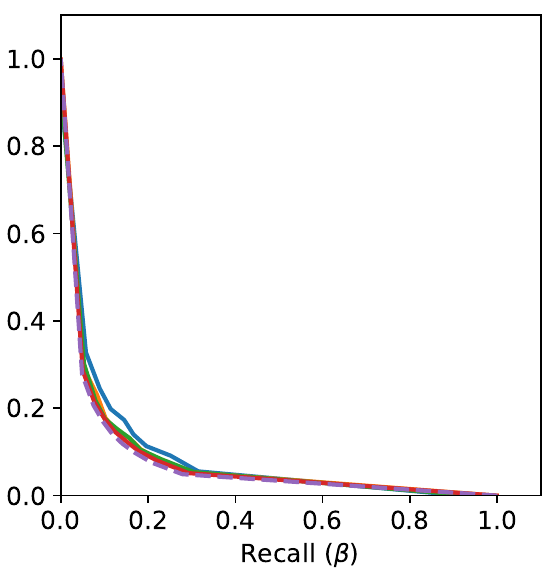}
        &
        \includegraphics[height=\myheight,trim={0 0 135 0},clip]{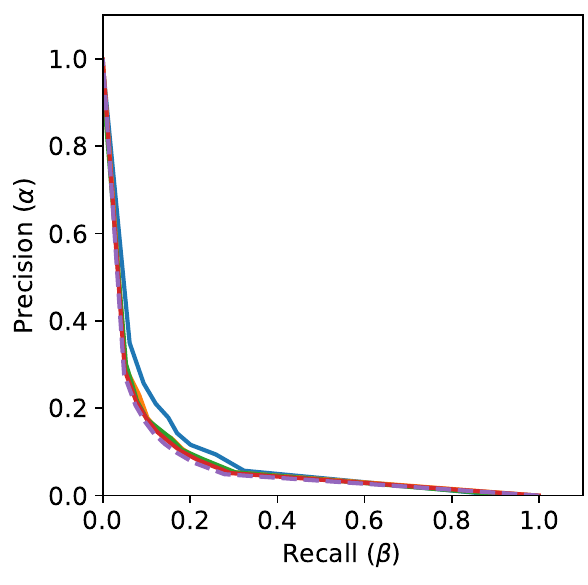}
         \\
        w/o outlier & w/ outlier
    \end{tabular}
    
    \caption{\textbf{Measuring the impact of an outlier.} 
    On the left, the setting is the same as Fig.~\ref{fig:shift-gauss}
    for a translation of $\mu=3/\sqrt d$ between two Gaussian in dimension $d=64$ (without splitting nor averaging).
    On the right, a single outlier ($x =  \mathbf{4}$) is added to the sample of $P$. 
    As reported in the literature, this affects the iPR classifier, yet the PR curves are barely affected by such a perturbation.
    }
    \label{fig:outlier-gauss}
\end{figure}

\textbf{Variability}
%
Fig.~\ref{fig:varibility-nbpoints} complements Section \ref{sec:exp_variability} about variability.
Average curves are obtained by computing the empirical mean of $N = 100$ PR curves obtained different random $n$-samples (with $n=10^4$), \emph{i.e.} 
 $$
    (\bar \alpha(\lambda) , \bar \beta(\lambda)) 
    = \frac{1}{n}\sum_{i=1}^n  (\alpha_i(\lambda) , \beta_i(\lambda)) 
 $$
Deviation from average curves are materialized with two curves
$$
 (\alpha_{\delta}(\lambda) , \beta_{\delta}(\lambda))
 =
    (\bar \alpha(\lambda), \bar \beta(\lambda)) + \delta(\lambda)  
$$
for $\delta = \pm \sigma$ with empirical estimator
$$
    \sigma^2 (\lambda) 
    = \frac{1}{n}\sum_{i=1}^n  \left( (\alpha_i(\lambda)-\bar \alpha(\lambda))^2 , (\beta_i(\lambda)- \bar \beta(\lambda))^2 \right).
$$

\setlength\myheight{0.21\textwidth} 
\begin{figure}[htb]
    \centering
    \begin{tabular}{cc}
        \includegraphics[height=\myheight]{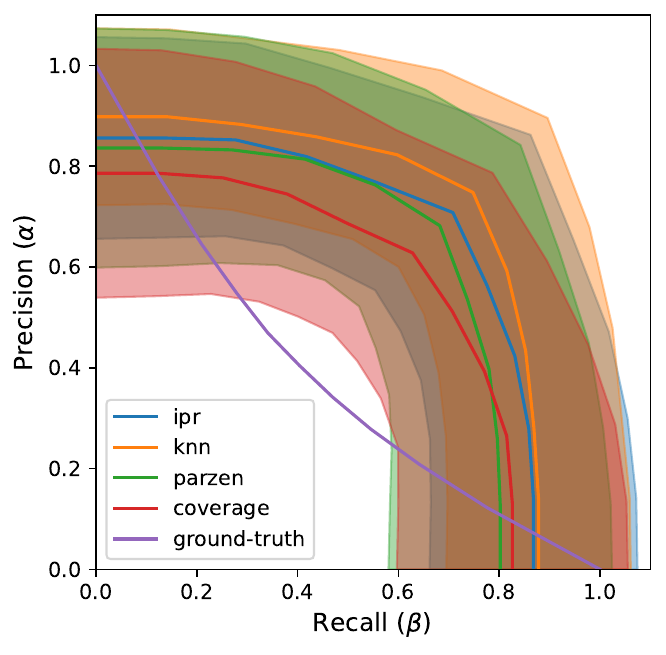}
        &
        \includegraphics[height=\myheight]{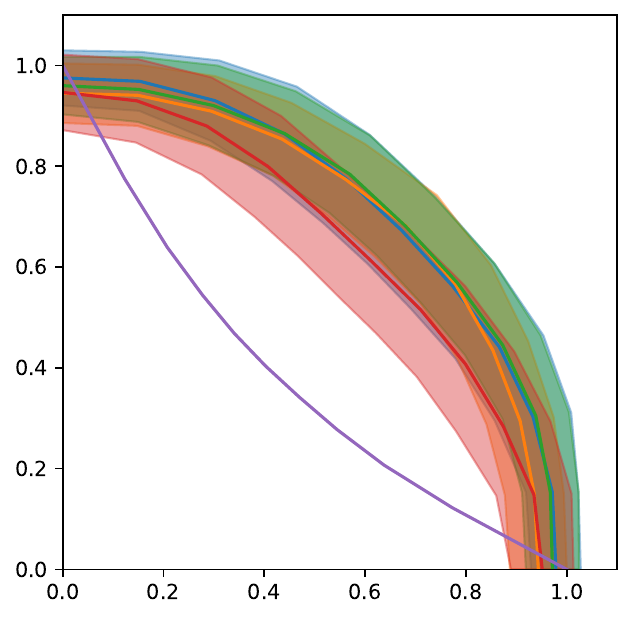}
        \\
        $n=10$ & $n=100$ \\
        \includegraphics[height=\myheight]{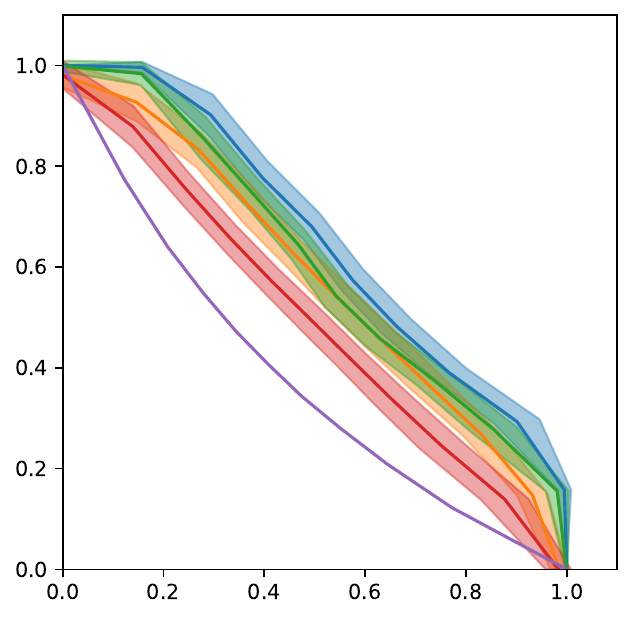}
        &
        \includegraphics[height=\myheight]{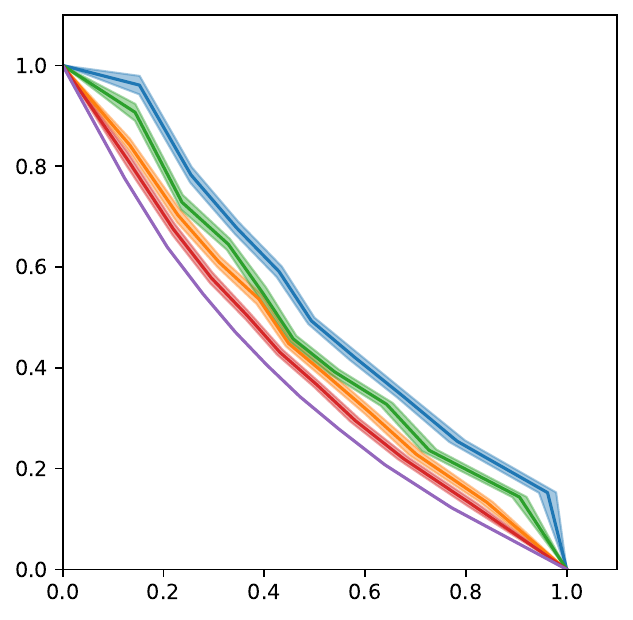}
         \\
        $n=1,000$ & $n=10,000$ 
    \end{tabular}
    \caption{\textbf{Influence of sample size $n$.}
    The setting is the same as Fig.~\ref{fig:shift-gauss}
    for a translation of $\mu=.21$ between two Gaussian in dimension $d=64$ (with splitting and $k=\sqrt n$).
    Solid (respectively transparent) curves correspond to the empirical average (resp. deviations) of $100$ PR curves computed from random samples. 
    (see the text for more details).
    }
    \label{fig:varibility-nbpoints}
\end{figure}